\documentclass[conference]{IEEEtran}
\IEEEoverridecommandlockouts

\usepackage{cite}
\usepackage{amsmath,amssymb,amsfonts}
\usepackage{algorithm,algpseudocode}
\usepackage{graphicx}
\usepackage{textcomp}
\usepackage{xcolor}

\usepackage{amsthm}
\usepackage{enumitem}
\usepackage{url}
\usepackage[colorlinks]{hyperref}
\usepackage{float, subfig, multirow, array, booktabs, wrapfig}
\graphicspath{{./figs/}}

\newtheorem{assumption}{Assumption}

\newtheorem{lemma}{Lemma}

\newtheorem{theorem}{Theorem}

\theoremstyle{definition}

\newcommand\red[1]{{\color{red}#1}}
\newcommand\blue[1]{{\color{blue}#1}}

\newcommand{\bfC}{\mathbf{C}}
\newcommand{\bfD}{\mathbf{D}}
\newcommand{\bfI}{\mathbf{I}}
\newcommand{\bfx}{\mathbf{x}}
\newcommand{\bfg}{\mathbf{g}}
\newcommand{\bfeps}{\boldsymbol{\epsilon}}
\newcommand{\bfsigma}{\boldsymbol{\sigma }}
\newcommand{\real}{\mathbb{R}}
\DeclareMathOperator*{\argmax}{argmax}

\DeclareMathOperator{\argtop}{argtop}
\DeclareMathOperator{\diag}{diag}
\definecolor{mygray}{gray}{0.5}
\algrenewcommand{\algorithmiccomment}[1]{// #1}
\newcommand\gray[1]{{\color{mygray}#1}}
\allowdisplaybreaks
\begin{document}

\title{Preconditioned Sharpness-Aware Minimization: \\
Unifying Analysis and a Novel Learning Algorithm
\thanks{*Equal contribution. Work was supported in part by NSF grants 2312547, 2212318, and 2126052.}
}

\author{\IEEEauthorblockN{Yilang Zhang*}
\IEEEauthorblockA{\textit{Dept. of ECE, Univ. of Minnesota}\\
Minneapolis, MN 55455, USA \\
zhan7453@umn.edu}
\and
\IEEEauthorblockN{Bingcong Li*}
\IEEEauthorblockA{\textit{Dept. of CS, ETH Z\"urich}\\
8092 Z\"urich, Switzerland \\
bingcong.li@inf.ethz.ch}
\and
\IEEEauthorblockN{Georgios B. Giannakis}
\IEEEauthorblockA{\textit{Dept. of ECE, Univ. of Minnesota}\\
Minneapolis, MN 55455, USA \\
georgios@umn.edu}
}

\maketitle

\begin{abstract}
Targeting solutions over `flat' regions of the loss landscape, sharpness-aware minimization (SAM) has emerged as a powerful tool to improve generalizability of deep neural network based learning. While several SAM variants have been developed to this end, a unifying approach that also guides principled algorithm design has been elusive. This contribution leverages preconditioning (pre) to unify SAM variants and provide not only unifying convergence analysis, but also valuable insights. Building upon preSAM, a novel algorithm termed infoSAM is introduced to address the so-called adversarial model degradation issue in SAM by adjusting gradients depending on noise estimates. Extensive numerical tests demonstrate the superiority of infoSAM across various benchmarks.
\end{abstract}


\section{Introduction}
Advances in deep neural network (DNN) architectures have led to impressive success across various domains including language, audio, and vision~\cite{devlin2018bert,gong2021ast,sam4vit}. Owing to the markedly high dimensionality, DNNs can memorize a large gamut of training data~\cite{zhang2016}. As a result, small loss during training does not guarantee generalization to unseen data. Catalyzing generalizability of DNNs through fine-grained training schemes remains a long-standing and prohibitively critical challenge. 

Popular approaches to cope with generalization include data augmentation and regularization~\cite{dropout2014,cubuk2019}. 
Albeit effective, data augmentation is often picked in a handcrafted way, and may not universally fit various models and datasets. This prompts combining augmentations, but the optimal mix requires extensive trials. On the other hand, regularization methods such as weight decay and dropout, though straightforward to use, largely omit data properties. For complex models, simply stacking data augmentation and regularization is insufficient~\cite{sam4vit}. In image classification for example, optimal augmentation and regularization can be class dependent~\cite{balestriero2022effects}. These limitations unveil the need for fine-grained approaches that jointly account for data and model characteristics. 

One such approach resorts to advanced optimization by carefully accounting for the loss landscape, which depends on latent properties of both data distribution, and the DNN architecture. Among possible solutions on the loss curve, those lying on a flatter valley have higher potential for improving generalizability~\cite{keskar2016}. Supporting evidence comes from theoretical analysis~\cite{maksym2022,wen2023} and empirical experimentation~\cite{jiang2020}. The resultant sharpness-aware minimization (SAM)~\cite{foret2021}, seeks a flatter region by forcing the surrounding neighborhood to have small loss. 
Various approaches have also been proposed to further boost the effectiveness of SAM~\cite{kwon2021,kim2022,zhuang2022,yang2022,liimplicit}. Unfortunately, a unifying framework is lacking to encompass existing SAM variants, and inspire the principled design of novel approaches. 

\begin{table}
	\caption{Existing instances of our preSAM framework}
	\centering
	\vspace{-0.2cm}
	\begin{tabular}{lcc}
	\toprule
	Approach                & Precond.             & Additional as.     \\
	\midrule
	ASAM~\cite{kwon2021}   & CP                        & non-divergent             \\
	FisherSAM~\cite{kim2022}   & CP                    & lo. bound stoch. grad.     \\
	$\ell_\infty$ or $\ell_1$-SAM~\cite{foret2021} &     OP                         & N/A  \\
	modified-SSAM~\cite{mi2022} &     OP                         & N/A              \\
	Lazy SAM~\cite{jiang2023,zhao222} & OP 		& N/A \\
    VaSSO~\cite{li2023} & OP & N/A \\
	\textbf{InfoSAM (ours)} &     OP                         & N/A                  \\
	\bottomrule
	\end{tabular}
	\label{tab.summary}
\end{table}

Toward this goal, the present work relies on \emph{preconditioning} to unify SAM variants; hence, the term preconditioned (pre) SAM. Depending on where preconditioning is effected, PreSAM categorizes existing SAM variants into objective/constraint preconditioning (OP/CP); see also Table~\ref{tab.summary}. 
Unifying convergence analysis will be offered for both CP and OP. This will not only establish convergence for e.g., ASAM~\cite{kwon2021} and FisherSAM~\cite{kim2022}, but will also guide the development of novel algorithms. Building on preSAM, a novel OP approach will be developed to ameliorate the effect of stochastic gradient noise that causes what is termed \emph{adversarial model degradation (AMD)}. This new approach, dubbed infoSAM, relies on a preconditioner that adjusts gradient entries depending on noise estimates, thus effectively bypassing the effect of gradient noise and leading to improved generalization. 
All in all, our contribution is three-fold.
\begin{enumerate}[leftmargin=.3cm]
\item[\textbullet] Rooted on preconditioning, a framework termed preSAM is developed to unify existing SAM variants, and categorize them as OP/CP according to their preconditioners. 
\item[\textbullet] PreSAM offers a unifying  convergence analysis for its two subcategories, which fulfills the missing analysis and unjustified experimental preferences of many SAM variants. 
\item[\textbullet] InfoSAM is our novel OP algorithm that handles the AMD issue in SAM. Numerical tests showcase the effectiveness of infoSAM in enhancing generalizability.
\end{enumerate}

\textbf{Notation}. Bold lowercase (capital) letters denote vectors (matrices); $\| \cdot \|$ and $\langle \cdot, \cdot \rangle$ stand for $\ell_2$-norm and inner product; $\mathbb{KL}(\cdot || \cdot)$ represents the KL divergence; and $\mathbf{e}_i \in \real^d$ is the $i$th column of the identity matrix $\bfI_d \in \real^{d \times d}$.

\section{SAM Recap}
Let $\bfx \in \real^d$ denote the parameters of a DNN, and $f$ the nonconvex empirical risk (loss) given a dataset ${\cal D}:=\{\mathbf{a}_i, b_i\}_{i=1}^D$ with feature $\mathbf{a}_i$ and label $b_i$. To find a solution lying in a flat basin of $f$, SAM enforces small loss on the neighborhood of $\bfx$. This is achieved by the minimax problem
\begin{align}\label{eq.prob}
	\min_{\bfx} \max_{\| \bfeps \| \le \rho} f \big(\bfx + \bfeps \big)
\end{align}
where $\bfx + \bfeps$ acts as the most ``adversarial'' model in the neighborhood sphere of radius $\rho$. The highly-nonconvex nature of \eqref{eq.prob} discourages solving the inner maximization exactly. SAM effects this using two approximations
\begin{align}\label{eq.sam_epsilon_full}
	\bfeps_t  = \argmax_{\| \bfeps \| \le \rho} f(\bfx_t + \bfeps) 
	& \overset{(a)}{\approx} \argmax_{\| \bfeps \| \le \rho} f(\bfx_t) + \langle \nabla f( \bfx_t), \bfeps \rangle \nonumber \\	
	&\overset{(b)}{\approx} \argmax_{\| \bfeps \| \le \rho} \langle \bfg_t(\bfx_t), \bfeps \rangle  
\end{align}
where $(a)$ follows from a first-order Taylor expansion, and $(b)$ replaces the gradient $\nabla f(\bfx_t)$ with the stochastic gradient $\bfg_t(\bfx_t)$. For convenience, we will refer to~\eqref{eq.sam_epsilon_full}, as \textit{SAM subproblem}. The latter admits the closed-form solution
\begin{align}\label{eq.sam_epsilon}
	\bfeps_t = \rho \bfg_t(\bfx_t) / \| \bfg_t(\bfx_t) \|.
\end{align}
SAM then updates $\bfx_t$ using the stochastic gradient $\bfg_t (\bfx_t + \bfeps_t)$ at $\bfx_t + \bfeps_t$. The steps of SAM are listed under Alg.~\ref{alg.sam}. 

\section{Unifying SAM via preconditioning}
This section introduces a unifying approach to finding the adversarial model, where popular SAM variants are subsumed as special cases. All proofs are deferred to the Appendix. 

\subsection{Preconditioned SAM}
PreSAM leverages preconditioning to encompass several SAM variants, each with different preconditioners. In its most general form, \textbf{preSAM} finds $\bfeps_t$ by solving a preconditioned version of \eqref{eq.sam_epsilon_full}:
\begin{align}\label{eq.precondition_full}
	\textbf{PreSAM:}~~~ 
	\max_{\bfeps} \langle \bfC_t \bfg_t(\bfx_t), \bfeps  \rangle	 ~~\text{s.t.}~~ \| \bfD_t \bfeps \| \le \rho.
\end{align}
Here, $\bfC_t, \bfD_t \in \real^{d \times d} $ are preconditioners that alter the geometry of the SAM subproblem. In doing so, the adversarial model can be equipped with designable properties. 
In particular, $\bfC_t$ skews the direction of $\bfg_t(\bfx_t)$ in the objective, while $\bfD_t$ reshapes the constraint set. Both $\bfC_t$ and $\bfD_t$ can change over iterations, allowing preSAM to adapt to the local geometry for each $t$. The original SAM subproblem~\eqref{eq.sam_epsilon_full} can be recovered by simply fixing $\bfC_t= \bfD_t = \bfI_d$. Supposing for simplicity that $\bfD_t$ is invertible, preSAM also admits a closed-form solution
\begin{align}\label{eq.precondition_eps}
	\bfeps_t = \rho  \bfD_t^{-2} \bfC_t \bfg_t(\bfx_t) / \|\bfD_t^{-1} \bfC_t \bfg_t(\bfx_t) \|.
\end{align}

\begin{algorithm}[t]
    \caption{PreSAM} \label{alg.sam}
    \begin{algorithmic}[1]
    	\State \textbf{Initialize:} $\bfx_0, \rho$
    	\For {$t=0,\dots,T-1$}
    		\State Sample a minibatch $\mathcal{B}_t$
    		\State Denote the stochastic gradient on $\mathcal{B}_t$ as $\bfg_t(\cdot)$
			\State (\textbf{preSAM}) Find $\bfeps_t$ via a unified manner \eqref{eq.precondition_eps}. 
			 \newline
			\Comment \gray{\blue{\textbf{SAM}}: $\bfC_t=\bfD_t = \bfI_d$;} \gray{\red{\textbf{InfoSAM}}: $\bfC_t$ and $\bfD_t$ via~\eqref{eq.infosam_epsilon}}
			\State Calculate stochastic gradient $\bfg_t(\bfx_t + \bfeps_t)$
			\State Update model via $\bfx_{t+1} = \bfx_t - \eta \bfg_t(\bfx_t + \epsilon_t)$
		\EndFor
		\State \textbf{Return:} $\bfx_T$
	\end{algorithmic}
\end{algorithm}

Before delving into specific choices for $\bfC_t$ and $\bfD_t$ in existing algorithms, a natural question is whether the preconditioners conflict with finding a `good' solution of~\eqref{eq.prob}. The challenge arises from the fact that~\eqref{eq.precondition_full} is no longer obtained from Taylor's expansion of $f(\bfx_t + \bfeps)$. We answer this question under several standard assumptions for nonconvex optimization and SAM~\cite{ghadimi2013,bottou2018,mi2022,zhuang2022}. 

\begin{assumption}\label{as.1}
	 $f(\bfx)$ is lower bounded, i.e., $f(\bfx) \ge f^*,\forall \bfx$.
\end{assumption}
\begin{assumption}\label{as.2}
	$\bfg(\bfx)$ is $L$-Lipschitz, i.e., $\| \bfg(\bfx) - \bfg(\mathbf{y}) \| \le L \| \bfx  - \mathbf{y}\|, \forall \bfx, \mathbf{y}$.
\end{assumption}
\begin{assumption}\label{as.3}
	$\bfg(\bfx)$ is unbiased with bounded variance, i.e., $\mathbb{E} [\bfg(\bfx) | \bfx] = \nabla f(\bfx)$, and $\mathbb{E} [\| \bfg(\bfx) - \nabla f(\bfx) \|^2 | \bfx] \le \sigma^2$.
\end{assumption}

Under these mild assumptions, the unified convergence is established in the following theorem.
\begin{theorem}[Unified convergence]\label{thm.psam}
	Suppose As. \ref{as.1} -- \ref{as.3} hold.	Let $\eta_t \equiv \eta = \frac{\eta_0}{ \sqrt{T}} \le \frac{2}{3L}$, and $\rho = \frac{\rho_0}{\sqrt{T}}$. In addition, suppose $\| \bfD_t^{-1} \| \le D_0, \forall t$. Then, preSAM in Alg.~\ref{alg.sam} guarantees that 
\begin{align*}
		 &  \frac{1}{T} \sum_{t=0}^{T-1} \mathbb{E} \| \nabla f(\bfx_t )\|^2 \le \mathcal{O}\bigg( \frac{  f (\bfx_0) {-} f^*  }{ \eta_0 \sqrt{T}}  +  \frac{L\rho_0^2 D_0^2}{\eta_0 \sqrt{T}}  + \frac{L \eta_0 \sigma^2}{\sqrt{T}} \bigg), \\
		& {\frac{1}{T}}\sum_{t=0}^{T-1}\mathbb{E} \| \nabla f(\bfx_t {+} \bfeps_t)\|^2 \le \frac{2}{T}\sum_{t=0}^{T-1}\mathbb{E} \| \nabla f(\bfx_t )\|^2 + \frac{2 L^2 \rho_0^2 D_0^2}{T}.
	\end{align*}
\end{theorem}

Thm.~\ref{thm.psam} reveals that $\bfD_t$ has to be designed carefully to avoid slowing down convergence. In contrast, $\bfC_t$ is more flexible to choose as it does not explicitly influence the convergence rate, which is yet critical for generalization because it determines how powerful the adversarial model is. 

Next, we elaborate on choices of $\bfC_t$ and $\bfD_t$ to link preSAM to existing SAM variants. We will also dive deeper into their influences on convergence, which has been overlooked by existing works. Even though it is possible to jointly design $\bfC_t$ and $\bfD_t$, most SAM variants only work with a single preconditioner. Depending on whether $\bfC_t = \bfI_d$ or $\bfD_t = \bfI_d$, preSAM can be further categorized into constraint preconditioning (CP) and objective preconditioning (OP).

\subsection{Constraint preconditioning (CP)}
CP aims to alter the constraint geometry in~\eqref{eq.precondition_full}, where it keeps $\bfC_t = \bfI_d$, and designs $\bfD_t$ on demand. Essentially, $\bfD_t$ converts the $\ell_2$-norm ball $\{\bfeps: \| \bfeps \| \le \rho\}$ into an ellipsoid
Intuitively, this is helpful when knowing a priori that certain dimensions contribute more to the adversarial model. A caveat for designing $\bfD_t$ is that its inversion should be affordable; cf.~\eqref{eq.precondition_eps}. As a consequence, most existing CP approaches rely on diagonal $\bfD_t$, as discussed next.

\textbf{Scale-invariant adversarial model via CP.} It was pointed out in~\cite{dinh2017} that proper rescaling of NN weights does not change the loss function. This means there exist multiple adversarial models with the same loss, rendering the optimal one indistinguishable from the rest. ASAM~\cite{kwon2021} copes with this issue by rescaling the constraint set, which serves as a specific instance of CP. In its simplest form, ASAM adopts $\bfD_t = \diag ( |\bfx_t|^{-1})$, where $|\cdot|$ and $\cdot^{-1}$ are entry-wise operators. If $[\bfx_t]_i$ is small, ASAM tends to increase the perturbation $[\bfeps_t]_i$. 

\textbf{Fisher adversarial model via CP.} While SAM seeks $\bfeps_t$ within a Euclidean ball, this can be extended to more sophisticated spaces. For example, FisherSAM~\cite{kim2022} considers a ball induced by KL divergence, namely $\mathbb{E}_{\cal D} \big[ \mathbb{KL}( p(b_i|\mathbf{a}_i, \bfx_t + \bfeps ) || p(b_i|\mathbf{a}_i, \bfx_t) ) \big] \le \rho$. Modified with several approximations for computational efficiency, FisherSAM ends up with a specific form of CP, where $\bfD_t = \diag ( |\bfg_t|)$.

\textbf{CP can challenge convergence.} As stated in Thm.~\ref{thm.psam}, the convergence rate of CP critically depends on $D_0$. Unfortunately, both ASAM and FisherSAM are on the edge of divergence. For ASAM, it holds that $D_0 = \max_t \|\bfx_t \|_{\infty}$, which could be unbounded unless assuming non-divergence. For FisherSAM, $D_0 = \max_t \|\bfg_t^{-1}\|_{\infty}$ can also be unbounded and slowdowns convergence as $[\bfg_t]_i$ can be arbitrarily small.

Moreover, for CP to attain the same convergence rate as SAM, it requires $\rho = \rho_0/\sqrt{T} \propto 1/D_0$. Upon ASAM convergence, it typically holds that $D_0 < 1$. This explains the empirical observation that a larger $\rho$ helps ASAM to perform best~\cite{kwon2021}. The same was also corroborated in our experiments, where adopting the same $\rho$ as SAM degrades ASAM's performance. Somehow ironically, an enlarged $\rho$ makes the Taylor approximation $(a)$ in~\eqref{eq.sam_epsilon_full} inaccurate, which can weaken the adversarial model. This leads to another issue for CP, that is, to determine the best $\rho$ through extra effort.

\subsection{Objective preconditioning (OP)}
For the objective in~\eqref{eq.precondition_full}, OP fixes $\bfD_t = \bfI_d$, and adapts merely $\bfC_t$. As asserted by Thm.~\ref{thm.psam}, OP is more flexible since convergence rate is not explicitly dependent on its preconditioner. In addition, OP is less stringent than CP because: i) $\bfC_t$ need not be invertible; and ii) scaling $\bfC_t$ has no impact on $\bfeps_t$. The latter can be verified by replacing $\bfC_t$ with $\alpha \bfC_t, \forall \alpha >0$, which does not alter the solution~\eqref{eq.precondition_eps}. By redirecting $\bfg_t(\bfx_t)$, OP seeks an improved adversarial model. Depending on the specific $\bfC_t$, OP can be used for various purposes. 

\textbf{Adversarial models in non-ellipsoidal neighborhood via OP.} While CP's constraint set is an ellipsoid, OP gives rise to a non-ellipsoidal neighborhood when $\bfC_t$ is properly designed. Table~\ref{tab.op_norm} exemplifies three choices of $\bfC_t$ for which the resultant $\bfeps_t$ amounts to solving~\eqref{eq.precondition_full} under $\ell_1$, $\ell_\infty$, or $n$-support norm ball~\cite{argyriou2012} constraints. The former two are found in~\cite{foret2021}, while the last is our extension, where a $n$-support norm ball can be viewed as a combination of $\ell_1$ and $\ell_2$ norm constraint. 

\textbf{Sparse perturbation via OP.} The second and third method in Table~\ref{tab.op_norm} both result in a sparse $\bfeps_t$. This helps reduce the backpropagation complexity of $\bfg_t(\bfx_t+ \bfeps_t)$. More involved approaches along this line include SSAM~\cite{mi2022}, which not only assumes bounded gradient, but also suffers from rate slower than SAM. These issues can be addressed by changing the algorithmic order; that is, first sparsify the gradient via OP by setting the corresponding entries of $\bfC_t$ to $0$ as~\cite[Alg. 2]{mi2022}, and then use infoSAM~\eqref{eq.infosam_epsilon} to obtain $\bfeps_t$. We term this method modified SSAM, and our experiments show that it matches the performance of vanilla SSAM. 

\begin{table}[t]
	\centering
	\tabcolsep=0.075cm
	\caption{OP and its equivalent constraint.}
	\vspace{-.2cm}
	\begin{tabular}{cc}
	\toprule
	OP    & Equiv. constr. for \eqref{eq.precondition_full}                   \\
	\midrule 
	$\bfC_t = \diag( |\bfg_t|^{-1})$         & $\| \bfeps \|_\infty \le \rho$                       \\
	$\bfC_t = \diag(\mathbf{e}_i)$ with $i = \argmax | [\bfg_t(\bfx_t)]_i |$   &  $\| \bfeps \|_1 \le \rho$                                      \\
	$\bfC_t = \diag(\sum_{i\in \mathcal{I}} \mathbf{e}_i)$ with $\mathcal{I} = \argtop_n(|\bfg_t(\bfx_t)|)$         & $\| \bfeps \|_\text{n-supp} \le \rho$                                  \\
	\bottomrule
	\end{tabular}
	\label{tab.op_norm}
\end{table}

\textbf{Lazy adversary model via OP.} Lazy SAM~\cite{jiang2023,zhao222} switches between SAM's adversarial objective~\eqref{eq.prob} and empirical risk minimization (ERM) to lower the computational cost. With ERM-induced update $\bfx_{t+1} = \bfx_t - \eta \bfg_t(\bfx_t)$, this avoids SAM's second gradient computation $\bfg(\bfx_t + \bfeps_t)$. Given that $\bfC_t = \mathbf{0}$ in~\eqref{eq.precondition_full} leads to $\bfeps_t = \mathbf{0}$, preSAM is able to recover lazy SAM by setting $\bfC_t = \mathbf{0}$ whenever switching to ERM. 

\textbf{Chain of preconditioners.} It is also possible to equip an adversarial model with multiple desired properties through a cascade of preconditioners. For example, if $\{ \bfC_{t,i} \}_{i=1}^I$ are valid OP choices, $\bfC_t = \prod_{i=1}^I \bfC_{t,i}$ is also a valid OP preconditioner. 


\section{InfoSAM}
This section develops a new instance of preSAM that copes with the adversarial model degradation challenge of SAM.

\subsection{Adversarial model degradation (AMD)}
The stochastic noise in $\bfg_t(\bfx_t)$ can markedly harm the adversarial model $\bfx_t + \bfeps_t$ obtained via~\eqref{eq.sam_epsilon}~\cite{li2023}. We term this \textit{adversarial model degradation},  
and further elaborate on its harmfulness, which motivates our novel algorithm, infoSAM.

Consider SAM in the ideally noise-free case, i.e., $\bfg_t(\bfx_t) = \nabla f(\bfx_t)$. Then, the perturbation of the $i$th dimension satisfies $[\bfeps_t]_i \propto [\nabla f(\bfx_t)]_i$; cf. \eqref{eq.sam_epsilon}. This matches the intuition for finding the most adversarial model, since it holds that 
\begin{align}\label{eq.ideal_sam}
	f(\bfx_t + \lambda & \mathbf{e}_i) - f(\bfx_t)  \leq \lambda \langle \nabla f(\bfx_t), \mathbf{e}_i \rangle + \frac{L\lambda^2}{2} \\
	& \stackrel{(a)}{=} \alpha [\nabla f(\bfx_t)]_i^2 + \frac{L \alpha^2}{2} [\nabla f(\bfx_t)]_i^2 \propto [\nabla f(\bfx_t)]_i^2 \nonumber 
\end{align}
where $(a)$ is by taking $\lambda = \alpha [\nabla f(\bfx_t)]_i$ for some $\alpha > 0$. When $[\nabla f(\bfx_t)]_i$ is large, the adversarial model has the potential to induce a higher loss by moving more toward this dimension.

In practice, SAM relies on $\bfg_t(\bfx_t)$ rather than $\nabla f(\bfx_t)$, with which~\eqref{eq.ideal_sam} can hardly hold. When the stochastic noise is dominant,  $[\nabla \bfg(\bfx_t)]_i$ can even correspond to a descent direction. When training a ResNet-18 on CIFAR10, we observed that the signal-to-noise ratio (SNR) is around $\mathcal{O}(10^{-2})$ throughout 200 training epochs. This suggests that the gradient noise is indeed a severe issue for SAM. Additional examples on how AMD affects the convergence behavior of SAM in an asymmetric valley can be found in App.~\ref{sec.apdx.AMD}. 

\subsection{A novel OP approach to handle AMD}
Unfortunately, no preSAM approach is available to deal with the AMD challenge caused by gradient noise. This section develops such an OP-based algorithm that we term infoSAM. 

Our conception of infoSAM is straightforward -- when seeking the adversarial model, we should be more cautious on dimensions with smaller SNR since they are less informative. Quantitatively, with $[\bfsigma_t]_i^2$ denoting the variance of $[\bfg_t (\bfx_t)]_i$, infoSAM's perturbation is $ [\bfeps_t]_i \propto [\bfg_t (\bfx_t)]_i/[\bfsigma_t]_i^2$. App.~\ref{sec.apdx.infosam} details how infoSAM works using a numerical case study. 

While alleviating AMD using $[\bfeps_t]_i \propto [\bfg_t (\bfx_t)]_i/[\bfsigma_t]_i^2$ is intriguing, the variance vector $\bfsigma_t^2$ is generally intractable. Inspired by~\cite{zhuang2020adabelief}, we estimate $\bfsigma_t^2$ by the squared difference between $\bfg_t(\bfx_t)$'s exponentially moving average (EMA) and $\bfg_t(\bfx_t)$ itself. The EMA 
$\mathbf{m}_t$ is accumulated as
\begin{subequations}\label{eq.infosam_epsilon}
\vspace{-0.15cm}
\begin{equation}
	\mathbf{m}_t = \alpha 	\mathbf{m}_{t-1} + (1-\alpha) \bfg_t(\bfx_t)
\end{equation}
where $0<\alpha < 1$ is a hyperparameter. Vector  $\mathbf{m}_t$ serves as an estimate of $\nabla f(\bfx_t)$, which is then leveraged to estimate
\begin{equation}
	\hat{\bfsigma}_t^2 = \big(\mathbf{m}_t -  \bfg_t(\bfx_t)\big)^2.
\end{equation}
With $\hat{\boldsymbol{\Sigma}}_t := \diag(\hat{\bfsigma}_t^2)$, infoSAM obtains its $\bfeps_t$ via
\begin{equation}
	\bfeps_t = \argmax_{\| \bfeps \| \leq \rho} \langle  \hat{\boldsymbol{\Sigma}}_t^{-1}\bfg_t(\bfx_t), \bfeps \rangle = \rho \frac{\hat{\boldsymbol{\Sigma}}_t^{-1}\bfg_t(\bfx_t)}{ \| \hat{\boldsymbol{\Sigma}}_t^{-1}\bfg_t(\bfx_t) \| }.
\end{equation}
\end{subequations}

The step-by-step implementation of infoSAM is summarized in Alg.~\ref{alg.sam}. It is also worth noting that infoSAM can be used jointly with CP methods such as ASAM and FisherSAM, which has been added to our future research agenda. 

\begin{table}[t]
	\caption{Comparison of infoSAM against other baselines.}
	\vspace{-0.2cm}
	\centering
	\tabcolsep=0.16cm
	\begin{tabular}{cccccc}
	\toprule
	 & Architecture          & SGD                & SAM                & ASAM               & InfoSAM    
	 \\
	\midrule
	\multirow{4}{*}{\rotatebox{90}{~\textbf{CIFAR10}}} 
	& ResNet              & 96.25$_{\pm0.06}$  & 96.58$_{\pm0.10}$  & 96.33$_{\pm0.09}$  & \textbf{96.71}$_{\pm0.09}$ \\
	&DenseNet            & 96.65$_{\pm0.13}$  & 96.94$_{\pm0.11}$  & 96.73$_{\pm0.18}$  & \textbf{97.09}$_{\pm0.07}$ \\
	&WideResNet        & 97.08$_{\pm0.16}$  & 97.32$_{\pm0.11}$  & 97.15$_{\pm0.05}$  & \textbf{97.56}$_{\pm0.12}$  \\
	&PyramidNet          & 97.39$_{\pm0.09}$  & 97.85$_{\pm0.14}$  & 97.56$_{\pm0.11}$  & \textbf{98.04}$_{\pm0.06}$  \\
	\midrule 
	\multirow{4}{*}{\rotatebox{90}{~\textbf{CIFAR100}}} 
	& ResNet              & 77.90$_{\pm0.07}$  & 80.96$_{\pm0.12}$  & 79.91$_{\pm0.04}$  & \textbf{81.31}$_{\pm0.15}$  \\
	& DenseNet           & 81.62$_{\pm0.19}$  & 83.94$_{\pm0.08}$  & 82.75$_{\pm0.10}$  & \textbf{84.09}$_{\pm0.12}$ \\
	&WideResNet        & 81.71$_{\pm0.13}$  & 84.88$_{\pm0.10}$  & 83.54$_{\pm0.14}$  & \textbf{85.01}$_{\pm0.07}$  \\
	& PyramidNet         & 83.50$_{\pm0.12}$  & 85.60$_{\pm0.11}$  & 83.72$_{\pm0.09}$  & \textbf{85.83}$_{\pm0.11}$  \\
	\bottomrule
	\end{tabular}
	\label{tab.cifar}
\end{table}

\section{Numerical tests}\label{sec.numerical}
Here we test infoSAM's numerical efficiency. Implementation details are deferred to App.~\ref{apdx.num}.

\subsection{CIFAR10 and CIFAR100}

The evaluation starts with image classification on benchmarks CIFAR10 and CIFAR100~\cite{krizhevsky2009learning}. The backbone architectures are convolutional neural networks including ResNet-18~\cite{He_2016_CVPR}, DenseNet-121~\cite{Huang_2017_CVPR}, WideResNet-28-10~\cite{zagoruyko2016wide}, and PyramidNet-110~\cite{Han_2017_CVPR}. Besides infoSAM, we also test stochastic gradient descent (SGD), SAM, and ASAM as baselines. 

The test accuracies are gathered in Tab.~\ref{tab.cifar}. The proposed infoSAM achieves the highest accuracy in all model setups, validating that AMD can be alleviated through proper preconditioning. The results also suggest that CP can be delicate when $\rho$ is not chosen properly. As discussed after Thm.~\ref{thm.psam}, CP approaches such as ASAM rely on a large $\rho$ to achieve comparable performance over SAM. This matches the results in Tab.~\ref{tab.cifar}, where ASAM underperforms SAM when adopting the same $\rho$, and only slightly improves over SGD. This demonstrates that CP has to be used cautiously, and further justifies our preference of OP for tackling the AMD issue. 

\subsection{ImageNet}

\begin{figure}[t]
	\centering
		\subfloat[]{\includegraphics[width=.5\linewidth]{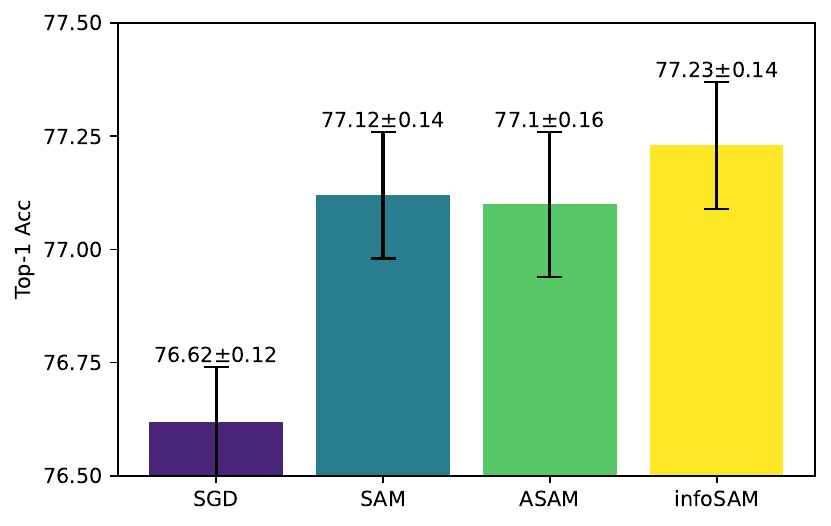}}
		\subfloat[]{\includegraphics[width=.5\linewidth]{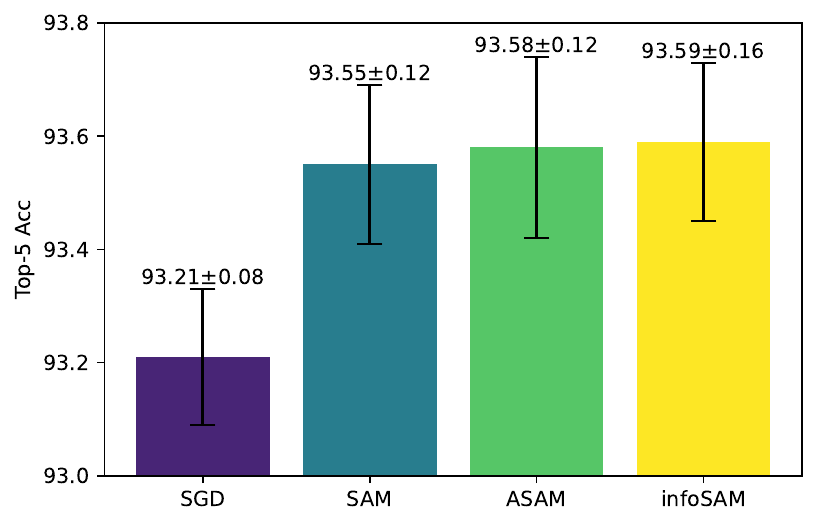}}
	\vspace{-.2cm}
	\caption{(a) Top-1 and (b) top-5 accuracies on ImageNet.}
	\label{fig.imagenet}
\end{figure}

Next, we investigate the performance of infoSAM on large-scale experiments by training a ResNet-50~\cite{He_2016_CVPR} on ImageNet~\cite{imagenet2009}. Fig.~\ref{fig.imagenet} plots the top-1 and top-5 accuracy of tested algorithms. It can be observed that infoSAM has the best top-1 as well as top-5 accuracies. 
Again, the CP-based ASAM does not catch up with SAM when using the same $\rho$.

\subsection{Label noise}

\begin{figure}
	\vspace{-0.3cm}
	\centering
		\includegraphics[width=.85\linewidth]{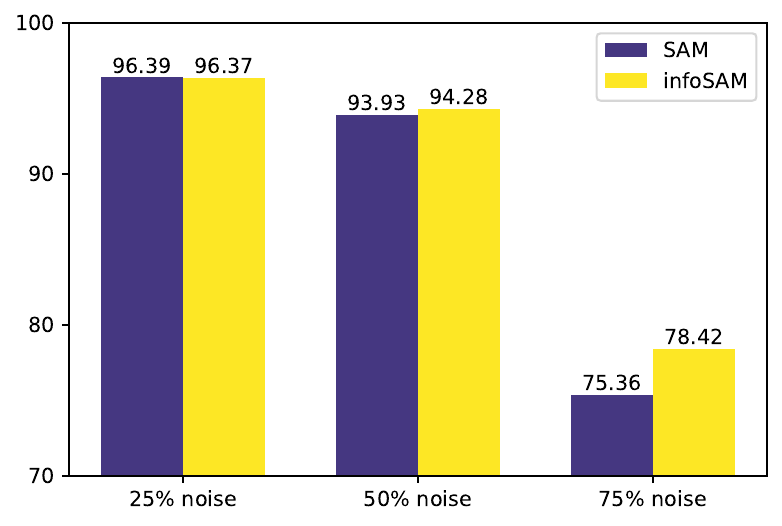}
	\vspace{-0.3cm}
	\caption{Performance under different levels of label noise.}
	 \label{fig.lbnoise}
\end{figure}

SAM is known to exhibit robustness against large label noise in the training set~\cite{foret2021}. Since the loss landscape can be heavily perturbed, it is expected that infoSAM outperforms SAM. In our experiments, we consider the classical noisy-label setting, where a fraction of the training labels are randomly flipped, whereas the test set remains clean. 
A ResNet-18~\cite{He_2016_CVPR} is trained on CIFAR10 with label noise levels $\{25\%, 50\%, 75\% \}$. It can be seen from Fig.~\ref{fig.lbnoise} that infoSAM markedly improves SAM in high-level label noise.  

\section{Conclusions}
We developed a preconditioning-based SAM framework that provides: i) unifying convergence analysis of SAM variants; ii) valuable insights of experimental results; and, iii) guidelines to develop novel SAM algorithms. Within this framework, infoSAM can tackle the AMD challenge of SAM, and thus improves generalization across various benchmarks.

\clearpage

\bibliography{datactr}
\bibliographystyle{IEEEtran}

\clearpage

\appendices
\setlength{\abovedisplayskip}{7pt}
\setlength{\belowdisplayskip}{7pt}

\section{Missing proofs}
For notational simplicity, we first rewrite Alg.~\ref{alg.sam} as
\begin{subequations}\label{eq.alg_rewrite}
\begin{align}
	\bfx_{t+\frac{1}{2}} &= \bfx_t + \bfeps_t,~~ \text{where} ~~\bfeps_t = \rho \frac{ \bfD_t^{-2} \bfC_t \bfg_t}{ \|\bfD_t^{-1} \bfC_t \bfg_t \|} \\
	 \bfx_{t+1} &= \bfx_t -  \eta_t \bfg_t(\bfx_{t+\frac{1}{2}}).
\end{align}
\end{subequations}
It follows that
\begin{align}\label{eq.apex.eps_norm_bound}
	\| \bfeps_t \| \leq \rho \| \bfD_t^{-1} \|	\frac{\| \bfD_t^{-1} \bfC_t \bfg_t \|}{ \|\bfD_t^{-1} \bfC_t \bfg_t \|} \leq \rho D_0.
\end{align}

Before presenting our proof, we first provide several useful lemmas to support the proof of our main results.

\subsection{Useful lemmas}

\begin{lemma}\label{apdx.lemma1}
	Alg.~\ref{alg.sam} (or equivalently iteration~\eqref{eq.alg_rewrite}) ensures that
	\begin{align*}
		&\eta_t \mathbb{E} \big[ \langle \nabla f(\bfx_t), \nabla f(\bfx_t) - \bfg_t(\bfx_{t+\frac{1}{2}})  \rangle \big] \le \nonumber \\
		&\hspace{3.5cm}\frac{L\eta_t^2 }{2} \mathbb{E} \big[ \| \nabla f(\bfx_t)\|^2 \big]  + \frac{LD_0^2\rho^2}{2}.
	\end{align*}
\end{lemma}
\begin{proof}
	To start with, we have that
	\begin{align*}
		&\big\langle \nabla f(\bfx_t), \nabla f(\bfx_t) - \bfg_t(\bfx_{t+\frac{1}{2}})  \big\rangle \nonumber \\
		&= \langle \nabla f(\bfx_t), \nabla f(\bfx_t) - \bfg_t(\bfx_t) + \bfg_t(\bfx_t) - \bfg_t(\bfx_{t+\frac{1}{2}})  \rangle.
	\end{align*}
	Taking expectation conditioned on $\bfx_t$, we arrive at
	\begin{align*}
		& \mathbb{E} \big[ \big\langle \nabla f(\bfx_t), \nabla f(\bfx_t) - \bfg_t(\bfx_{t+\frac{1}{2}})  \big\rangle | \bfx_t \big]  \\
		& = \mathbb{E} \big[ \langle \nabla f(\bfx_t), \nabla f(\bfx_t) - \bfg_t(\bfx_t)   \rangle | \bfx_t \big]  \\
		&\qquad+ \mathbb{E} \big[ \langle \nabla f(\bfx_t), \bfg_t(\bfx_t) - \bfg_t(\bfx_{t+\frac{1}{2}})  \rangle | \bfx_t \big]  \\
		& = \mathbb{E} \big[ \langle \nabla f(\bfx_t), \bfg_t(\bfx_t) - \bfg_t(\bfx_{t+\frac{1}{2}})  \rangle | \bfx_t \big]   \\
		& \leq \mathbb{E} \big[ \| \nabla f(\bfx_t)\|  \cdot \| \bfg_t(\bfx_t) - \bfg_t(\bfx_{t+\frac{1}{2}}) \| | \bfx_t \big]  \\
		& \stackrel{(a)}{\leq} L  \mathbb{E} \big[ \| \nabla f(\bfx_t)\| \cdot \| \bfx_t - \bfx_{t+\frac{1}{2}} \| | \bfx_t \big]  \\
		& \stackrel{(b)}{=}  L  \rho  D_0 \| \nabla f(\bfx_t)\| 
	\end{align*}
	where $(a)$ follows from As.~\ref{as.2}; and $(b)$ is because $\bfx_t - \bfx_{t+\frac{1}{2}} = -\bfeps_t$ and its norm is bounded by \eqref{eq.apex.eps_norm_bound}. 
	
	This inequality ensures that
	\begin{align*}
		&\eta_t \mathbb{E} \big[ \big\langle \nabla f(\bfx_t), \nabla f(\bfx_t) - \bfg_t(\bfx_{t+\frac{1}{2}})  \big\rangle | \bfx_t \big] \\
		&\leq  L D_0 \rho \eta_t \| \nabla f(\bfx_t)\| \\
		&\leq \frac{L \eta_t^2 \| \nabla f(\bfx_t)\|^2 }{2} + \frac{L D_0^2 \rho^2}{2}
	\end{align*}
	where the last inequality is because $\rho D_0 \eta_t \| \nabla f(\bfx_t)\| \leq \frac{1}{2} \eta_t^2 \| \nabla f(\bfx_t)\|^2 + \frac{1}{2} \rho^2 D_0^2 $. Taking expectation w.r.t. $\bfx_t$ finishes the proof.
\end{proof}

\begin{lemma}\label{apdx.lemma2}
	Alg.~\ref{alg.sam} (or equivalently iteration~\eqref{eq.alg_rewrite}) ensures that
	\begin{align*}
		\mathbb{E} \big[ \| \bfg_t(\bfx_{t+\frac{1}{2}}) \|^2 \big] \leq 2 L^2 D_0^2 \rho^2 + 2 \mathbb{E} \big[ \|  \nabla f(\bfx_t) \|^2 \big] + 2 \sigma^2.
	\end{align*}
\end{lemma}
\begin{proof}
	The proof starts with bounding $\| \bfg_t(\bfx_{t+\frac{1}{2}})\|$ via
	\begin{align*}
		\| \bfg_t(\bfx_{t+\frac{1}{2}}) \|^2 & = \| \bfg_t(\bfx_{t+\frac{1}{2}})  - \bfg_t(\bfx_t) + \bfg_t(\bfx_t) \|^2 \\
		& \leq 2 \| \bfg_t(\bfx_{t+\frac{1}{2}}) - \bfg_t(\bfx_t) \|^2 + 2 \| \bfg_t (\bfx_t) \|^2 \\
		& \stackrel{(a)}{\leq}  2 L^2 \|\bfx_t - \bfx_{t+\frac{1}{2}} \|^2 + 2 \| \bfg_t(\bfx_t) \|^2  \\
		& \stackrel{(b)}{=} 2 L^2 D_0^2 \rho^2 + 2 \| \bfg_t(\bfx_t) - \nabla f(\bfx_t) + \nabla f(\bfx_t) \|^2  
	\end{align*}
	where $(a)$ is the result of As.~\ref{as.2}; and $(b)$ is because $\bfx_t - \bfx_{t+\frac{1}{2}} = -\bfeps_t$ and its norm is bounded in \eqref{eq.apex.eps_norm_bound}. 
	
	Taking expectation conditioned on $\bfx_t$, we have
	\begin{align*}
		&\mathbb{E}\big[ \| \bfg_t( \bfx_{t+\frac{1}{2}}) \|^2 | \bfx_t \big]	 \\
		& \leq 2 L^2 D_0^2 \rho^2 + 2 \mathbb{E}\big[ \| \bfg_t(\bfx_t) - \nabla f(\bfx_t) + \nabla f(\bfx_t) \|^2 | \bfx_t \big] \nonumber \\
		& \leq 2 L^2 D_0^2 \rho^2 + 2 \|  \nabla f(\bfx_t) \|^2 + 2 \sigma^2 \nonumber
	\end{align*}
	where the last inequality is from As.~\ref{as.3}. Taking expectation w.r.t. the randomness of $\bfx_t$ finishes the proof.
\end{proof}


\subsection{Proof of Theorem~\ref{thm.psam}}
\begin{proof}
%
	Using As.~\ref{as.2}, we have that
	\begin{align*}
		&f(\bfx_{t+1}) - f (\bfx_t)  \\
		& \leq 	\langle \nabla f(\bfx_t), \bfx_{t+1} - \bfx_t \rangle + \frac{L}{2} \| \bfx_{t+1} - \bfx_t \|^2 \\
		& = - \eta_t \langle \nabla f(\bfx_t), \bfg_t(\bfx_{t+\frac{1}{2}})  \rangle + \frac{L\eta_t^2}{2} \| \bfg_t(\bfx_{t+\frac{1}{2}}) \|^2 \nonumber \\
		& = - \eta_t \langle \nabla f(\bfx_t), \bfg_t(\bfx_{t+\frac{1}{2}}) - \nabla f(\bfx_t) + \nabla f(\bfx_t)  \rangle + \\
		&\qquad \frac{L\eta_t^2}{2} \| \bfg_t(\bfx_{t+\frac{1}{2}})\|^2 \nonumber \\
		& = - \eta_t \| \nabla f(\bfx_t) \|^2 - \eta_t \langle \nabla f(\bfx_t), \bfg_t(\bfx_{t+\frac{1}{2}}) - \nabla f(\bfx_t)  \rangle + \\
		&\qquad \frac{L\eta_t^2}{2} \| \bfg_t(\bfx_{t+\frac{1}{2}}) \|^2 .
	\end{align*}
	
	Taking expectation, then plugging in Lemmas \ref{apdx.lemma1} and \ref{apdx.lemma2}, we have
	\begin{align*}
		\mathbb{E}\big[  f(\bfx_{t+1}) - f (\bfx_t)  \big] &\leq -  \bigg( \eta_t - \frac{3L\eta_t^2}{2} \bigg) \mathbb{E}\big[ \| \nabla f(\bfx_t )\|^2 \big] + \\
		&\qquad \frac{L\rho^2 D_0^2}{2}  + L^3 \eta_t^2 \rho^2 D_0^2 + L \eta_t^2 \sigma^2.
	\end{align*}

	As the parameter selection ensures that $\eta_t \equiv \eta = \frac{\eta_0}{ \sqrt{T}} \leq \frac{2}{3L}$, dividing both sides by $\eta$ and rearranging the terms give
	\begin{align*}
		&\bigg( 1 - \frac{3L\eta}{2} \bigg) \mathbb{E}\big[ \| \nabla f(\bfx_t )\|^2 \big]  \leq \frac{\mathbb{E}\big[  f (\bfx_t) - f(\bfx_{t+1})  \big]}{\eta} + \\
		&\hspace{3cm} \frac{L\rho^2 D_0^2}{2 \eta }  + L^3 \eta \rho^2 D_0^2 + L \eta \sigma^2.
	\end{align*}

	\begin{figure*}[t]
	\centering
	\includegraphics[width=.75\textwidth]{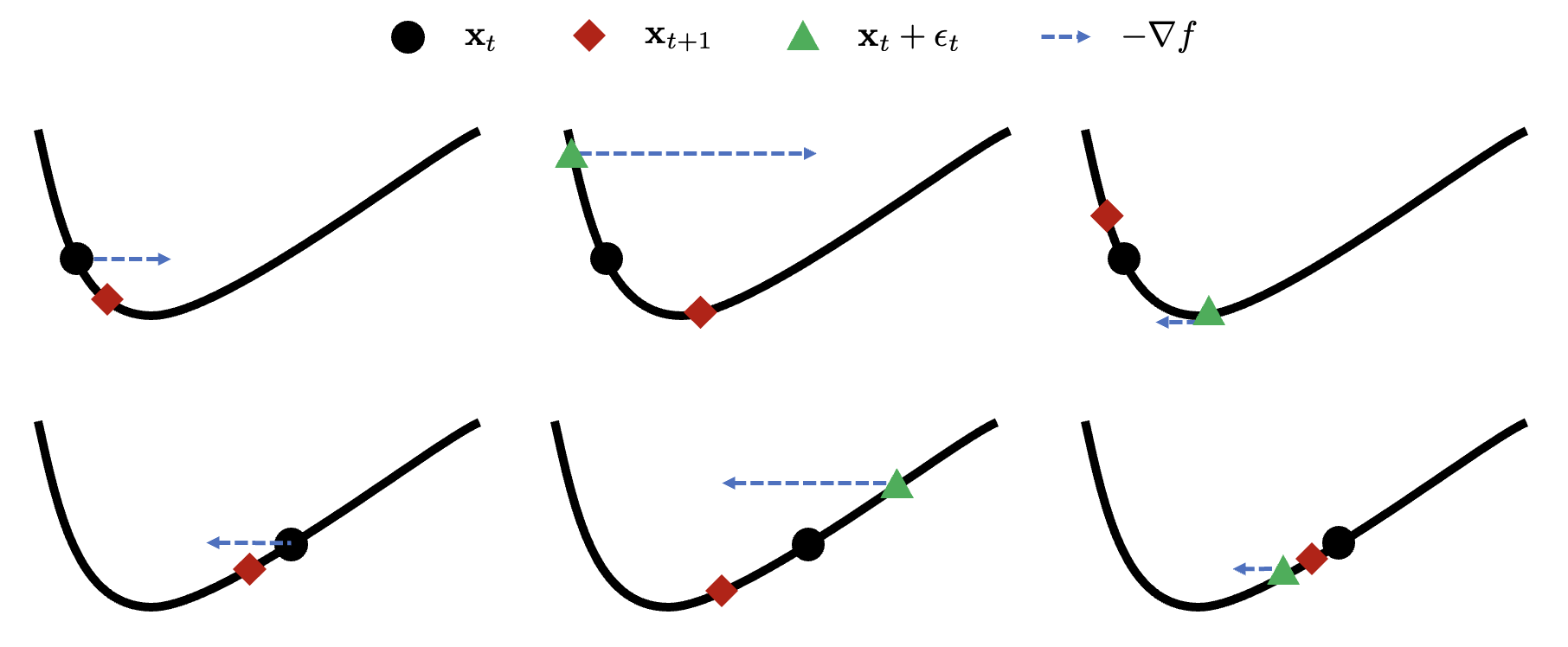}
	\caption{Behavior of SGD (left), ideal SAM (middle), and SAM with stochastic noise (right) near asymmetric valley. First row: transition from a sharper slope to a flatter one; second row: minimizing a flatter slope. Comparing middle with left reveals why SAM is helpful for finding a solution on flatter slope that generalizes better. The right part shows why gradient noise causes AMD.}
	 \label{fig.asym}
	 \vspace{-0.1cm}
\end{figure*}

	Summing over $t$, we have
	\begin{align*}
		&\bigg( 1 - \frac{3L\eta}{2} \bigg) \frac{1}{T}\sum_{t=0}^{T-1}\mathbb{E}\big[ \| \nabla f(\bfx_t )\|^2 \big] \\
		& \leq \frac{\mathbb{E}\big[  f (\bfx_0) - f(\bfx_T)  \big]}{\eta T}  + \frac{L\rho^2 D_0^2}{2 \eta }  + L^3 \eta \rho^2 D_0^2  + L \eta \sigma^2 \\
		& \stackrel{(a)}{\leq} \frac{  f (\bfx_0) - f^*  }{\eta T}  + \frac{L\rho^2 D_0^2}{2 \eta }  + L^3 \eta \rho^2 D_0^2 + L \eta \sigma^2 \nonumber \\
		& = \frac{  f (\bfx_0) - f^*  }{\eta_0 \sqrt{T}}  +  \frac{L\rho_0^2 D_0^2}{2 \eta_0 \sqrt{T}}  + \frac{L^3 \eta_0 \rho_0^2 D_0^2}{T^{3/2}}  + \frac{L \eta_0 \sigma^2}{\sqrt{T}} \nonumber 
	\end{align*}
	where (a) uses As.~\ref{as.1}, and the last equation is by plugging in the value of $\rho$ and $\eta$. This completes the proof to the first part. 
	
	For the second part of this theorem, we have
	\begin{align*}
		&\mathbb{E}\big[ \| \nabla f(\bfx_t + \bfeps_t )\|^2 \big] \\
		& = \mathbb{E}\big[ \| \nabla f(\bfx_t + \bfeps_t ) +\nabla f(\bfx_t) - \nabla f(\bfx_t) \|^2 \big] \\
		& \leq  2 \mathbb{E}\big[ \| \nabla f(\bfx_t \|^2 \big] + 2 \mathbb{E}\big[ \| \nabla f(\bfx_t + \bfeps_t ) - \nabla f(\bfx_t) \|^2 \big] \\
		& \leq  2 \mathbb{E}\big[ \| \nabla f(\bfx_t \|^2 \big] + 2 L^2 \rho^2 D_0^2 \\
		& = 2 \mathbb{E}\big[ \| \nabla f(\bfx_t \|^2 \big] +  \frac{2 L^2 \rho_0^2 D_0^2}{T}.
	\end{align*}
	Averaging over $t$ completes the proof.
\end{proof}

\section{VaSSO as an OP approach} \label{sec.apdx.VaSSO}
VaSSO in~\cite{li2023} can be also viewed as an objective preconditioning (OP) approach. Indeed, VaSSO acquires $\boldsymbol{\epsilon}_t$ via
\begin{align*}
    \mathbf{d}_t &= (1 - \theta) \mathbf{d}_{t-1} + \theta \mathbf{g}_t (\mathbf{x}_t) \\
    \boldsymbol{\epsilon}_t &= \mathop{\arg\max}_{\| \boldsymbol{\epsilon} \| \le \rho} f(\mathbf{x}_t) + \langle \mathbf{d}_t, \boldsymbol{\epsilon}_t \rangle = \rho \mathbf{d}_t / \| \mathbf{d}_t \|
\end{align*}
where $\mathbf{d}_t$ represents the running average of $\{ \mathbf{g}_\tau (\mathbf{x}_\tau) \}_{\tau=1}^t$. With OP having $\mathbf{D}_t = \mathbf{I}_d$, and 
\begin{equation*}
    \mathbf{C}_t = \frac{1-\theta}{\| \mathbf{g}_t (\mathbf{x}_t) \|^2} \mathbf{d}_{t-1} \mathbf{g}^\top_t(\mathbf{x}_t) + \theta \mathbf{I}_d 
\end{equation*}
it follows that 
\begin{equation*}
    \mathbf{C}_t \mathbf{g}_t (\mathbf{x}_t) = (1 - \theta) \mathbf{d}_{t - 1} + \theta \mathbf{g}_t (\mathbf{x}_t) = \mathbf{d}_t,
\end{equation*}
thus recovering the VaSSO method developed in~\cite{li2023}.

\section{Additional case study for AMD near an asymmetric valley}\label{sec.apdx.AMD}
AMD can be also observed when studying the convergence behavior of SAM near an asymmetric valley~\cite{he2019}. Simply put, an asymmetric valley is an area where the loss function grows at different rates at the positive and negative directions; see the black curve in Fig.~\ref{fig.asym}. Asymmetric valleys widely appear in the training loss of DNNs, where a solution biased toward the flatter slope can provably generalize better~\cite{he2019}. For the ease of illustration, we consider a one dimensional asymmetric valley while our arguments extends to more complicated cases. As shown in Fig.~\ref{fig.asym}, ideal SAM (without gradient noise) finds a desirable solution faster than SGD. In comparison, noisy SAM can significantly hurt the performance, as detailed in the following.

Consider the behavior of (ideal) SAM under two cases: i) transiting from sharper to flatter slope; and ii) minimizing the flatter slope. For case i), it can be observed that ideal SAM update employs gradient at an informative adversarial model, which is helpful to accelerate the transition from sharper slope to flatter one. This is not always true for the non-ideal SAM under gradient noise, as the adversarial model can have negative impact on moving to a flatter slope. In case ii), the flatter slope is not easy to be minimized since the gradient tends to have small magnitude here. Once again, ideal SAM accelerate this procedure by using a larger gradient at adversarial model; however, noisy SAM converges slowly when the gradient is perturbed to the negative direction due to the low SNR. 

\begin{figure}[t]
	\centering
		\subfloat[]{\includegraphics[width=.49\linewidth]{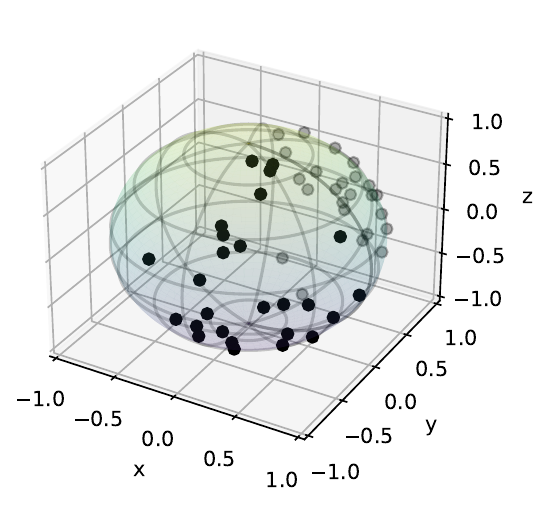}}
		\subfloat[]{\includegraphics[width=.49\linewidth]{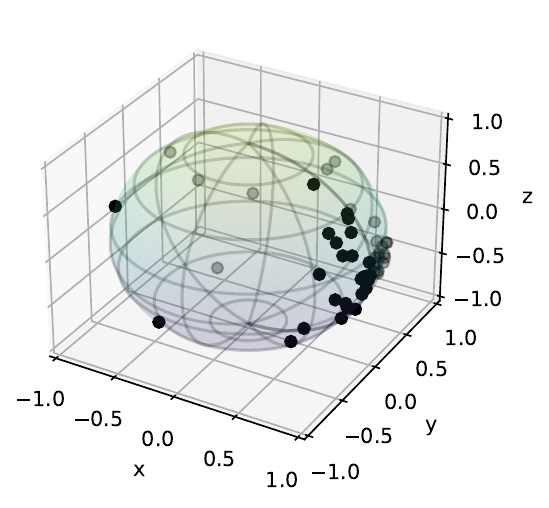}} 
	\caption{Comparison of the adversarial models in (a) SAM and (b) infoSAM.}
	 \label{fig.example}
\end{figure}

\section{Numerical examples for infoSAM}\label{sec.apdx.infosam}
To understand how infoSAM works, consider the case where $\nabla f(\bfx_t) = [0.2, -0.02, 0.01]$, and the stochastic gradient $\bfg_t(\bfx_t) = \nabla f(\bfx_t) + \boldsymbol{\xi}$. Let the stochastic noise $\boldsymbol{\xi}$ has a covariance matrix $\alpha \cdot \diag([0.2, 2, 1])$. We tune $\alpha$ so that the $\text{SNR}=0.1$. Without loss of generality, we assume $\bfx_t=\mathbf{0}$ so that the adversarial model is simply $\bfeps_t$. In the noise-free case, $\bfeps_t$ should be proportional to $\nabla f(\bfx_t)$, i.e., large magnitude in $x$-axis but small in $y$ and $z$ axises. With the gradient noise however, the corresponding $\bfeps_t$ obtained via SAM and infoSAM are plotted in Fig \ref{fig.example} (a) and (b), respectively. It can be observed that SAM is misled by the large noise on $y$- and $z$-axeses, and more than half of simulated $\bfeps_t$ are small on $x$-axis, suggesting a sever AMD issue. InfoSAM, on the contrary, generates $\bfeps_t$ concentrated around $[1, 0, 0]$, meaning that the information of x-axis is well captured. 

In the asymmetric valley example, although infoSAM will not entirely eliminate the AMD issue, it still cautiously finds an adversarial model. In other words, whenever the gradient noise is too large, the perturbation on the corresponding dimension will be inversely scaled with the variance of noise, so that infoSAM would not making large mistakes. 

\section{More details on numerical experiments}\label{apdx.num}

\subsection{CIFAR10 and CIFAR100}
For data augmentation, standard implementation including random crop, random horizontal flip, normalization and cutout \cite{cutout2017} are leveraged. Hyperparameters used in our experiments are summarized in Tabs.~\ref{tab.hyper-cifar10} and~\ref{tab.hyper-cifar100}. 

 \begin{table}[t]
 \tabcolsep=0.15cm
\caption{Hyperparameters for training from scratch on CIFAR10}
    \centering
   \small
    \begin{tabular}{ccccc}
    \toprule
      \textbf{ResNet-18}   & SGD & SAM & ASAM &infoSAM\\
      \midrule
      epoch&\multicolumn{4}{c}{200}\\
      batch size&\multicolumn{4}{c}{256}\\
      initial learning rate &\multicolumn{4}{c}{0.1}\\
      learning rate decay &\multicolumn{4}{c}{cosine}\\
      weight decay &$5\times10^{-4}$&$1\times10^{-3}$&$1\times10^{-3}$&$1\times10^{-3}$\\
      $\rho$   &-  &0.1&0.1&0.1\\
      $\alpha$ & - & - & - & 0.05 \\
      \midrule
      \textbf{DenseNet-121}   & SGD & SAM &ASAM & infoSAM\\
      \midrule
      epoch&\multicolumn{4}{c}{200}\\
      batch size&\multicolumn{4}{c}{256}\\
      initial learning rate &\multicolumn{4}{c}{0.1}\\
      learning rate decay &\multicolumn{4}{c}{cosine}\\
      weight decay &$5\times10^{-4}$&$1\times10^{-3}$&$1\times10^{-3}$&$5\times10^{-4}$\\
      $\rho$&-&0.1&0.1& 0.1\\
      $\alpha$ & - & - & - & 0.01 \\
      \midrule
      \textbf{WRN-28-10}   & SGD & SAM &ASAM & infoSAM\\
      \midrule
      epoch&\multicolumn{4}{c}{200}\\
      batch size&\multicolumn{4}{c}{256}\\
      initial learning rate &\multicolumn{4}{c}{0.1}\\
      learning rate decay &\multicolumn{4}{c}{cosine}\\
      weight decay &$5\times10^{-4}$&$1\times10^{-3}$&$1\times10^{-3}$&$5\times10^{-4}$\\
      $\rho$&-&0.1&0.1& 0.1\\
      $\alpha$ & - & - & - & 0.05 \\
	  \midrule
      \textbf{PyramidNet-110}   & SGD & SAM &ESAM\\
      \midrule
      epoch&\multicolumn{4}{c}{300}\\
      batch size&\multicolumn{4}{c}{128}\\
      initial learning rate &\multicolumn{4}{c}{0.05}\\
      learning rate decay &\multicolumn{4}{c}{cosine}\\
      weight decay &$5\times10^{-4}$&$1\times10^{-3}$&$1\times10^{-3}$&$5\times10^{-4}$\\
      $\rho$&-&0.1&0.1& 0.2\\
      $\alpha$ & - & - & - & 0.05 \\
        \bottomrule
    \end{tabular}
    \label{tab.hyper-cifar10}
\end{table}

 \begin{table}[t]
 \tabcolsep=0.15cm
\caption{Hyperparameters for training from scratch on CIFAR100}
    \centering
   \small
    \begin{tabular}{c|cccc}
    \toprule
      \textbf{ResNet-18 }   & SGD & SAM & ASAM &infoSAM\\
      \midrule
      epoch&\multicolumn{4}{c}{200}\\
      batch size&\multicolumn{4}{c}{256}\\
      initial learning rate &\multicolumn{4}{c}{0.1}\\
      learning rate decay &\multicolumn{4}{c}{cosine}\\
      momentum &\multicolumn{4}{c}{0.9} \\
      weight decay &$5\times10^{-4}$&$1\times10^{-3}$&$1\times10^{-3}$&$1\times10^{-3}$\\
      $\rho$   &-  &0.2&0.2&0.2\\
      $\alpha$ & - & - & - & 0.025 \\
      \midrule
      \textbf{DenseNet-121}   & SGD & SAM &ASAM & infoSAM\\
      \midrule
      epoch&\multicolumn{4}{c}{200}\\
      batch size&\multicolumn{4}{c}{256}\\
      initial learning rate &\multicolumn{4}{c}{0.1}\\
      learning rate decay &\multicolumn{4}{c}{cosine}\\
      momentum &\multicolumn{4}{c}{0.9} \\
      weight decay &$5\times10^{-4}$&$1\times10^{-3}$&$1\times10^{-3}$&$5\times10^{-4}$\\
      $\rho$&-&0.2&0.2& 0.2\\
      $\alpha$ & - & - & - & 0.001 \\
      \midrule
      \textbf{WRN-28-10}   & SGD & SAM &ASAM & infoSAM\\
      \midrule
      epoch&\multicolumn{4}{c}{200}\\
      batch size&\multicolumn{4}{c}{256}\\
      initial learning rate &\multicolumn{4}{c}{0.1}\\
      learning rate decay &\multicolumn{4}{c}{cosine}\\
      momentum &\multicolumn{4}{c}{0.9} \\
      weight decay &$5\times10^{-4}$&$1\times10^{-3}$&$1\times10^{-3}$&$5\times10^{-4}$\\
      $\rho$&-&0.2&0.2& 0.2\\
      $\alpha$ & - & - & - & 0.025 \\
	  \midrule
      \textbf{PyramidNet-110}   & SGD & SAM &ASAM & infoSAM\\
      \midrule
      epoch&\multicolumn{4}{c}{300}\\
      batch size&\multicolumn{4}{c}{128}\\
      initial learning rate &\multicolumn{4}{c}{0.05}\\
      learning rate decay &\multicolumn{4}{c}{cosine}\\
      momentum &\multicolumn{4}{c}{0.9} \\
      weight decay &$5\times10^{-4}$&$1\times10^{-3}$&$1\times10^{-3}$&$5\times10^{-4}$\\
      $\rho$&-&0.2&0.2& 0.2\\
      $\alpha$ & - & - & - & 0.001 \\
        \bottomrule
    \end{tabular}
    \label{tab.hyper-cifar100}
\end{table}

\subsection{ImageNet}
ImageNet~\cite{imagenet2009} has 1,281,167 images from 1000 classes for training and 50,000 images for validation. Due to the constraints on computational resources, we report the averaged results over $2$ independent runs. For this dataset, we randomly resize and crop all images to a resolution of $224\times 224$, and apply random horizontal flip, normalization during training. Hyperparameters for this dataset can be found in Tab.~\ref{tab.hyper-imagenet}. 

\begin{table}[t]
 \tabcolsep=0.15cm
\caption{Hyperparameters for training from scratch on ImageNet}
    \centering
   \small
    \begin{tabular}{c|cccc}
    \toprule
      \textbf{ResNet-18}   & SGD & SAM & ASAM &infoSAM\\
      \midrule
      epoch&\multicolumn{4}{c}{90}\\
      batch size&\multicolumn{4}{c}{128}\\
      initial learning rate &\multicolumn{4}{c}{0.05}\\
      learning rate decay &\multicolumn{4}{c}{cosine}\\
      momentum &\multicolumn{4}{c}{0.9} \\
      weight decay &$1\times10^{-4}$&$1\times10^{-4}$&$1\times10^{-4}$&$1\times10^{-4}$\\
      $\rho$   &-  &0.075&0.075&0.075\\
      $\alpha$ & - & - & - & 0.005 \\
        \bottomrule
    \end{tabular}
    \label{tab.hyper-imagenet}
\end{table}


\end{document}